\def\year{2020}\relax
\documentclass[letterpaper]{article} 
\usepackage{aaai20}  
\usepackage{times}  
\usepackage{helvet} 
\usepackage{courier}  
\usepackage[hyphens]{url}  
\usepackage{graphicx} 
\usepackage{amsmath,amsthm,amssymb,bm, url,mathtools,algorithmic,algorithm}
\usepackage[utf8]{inputenc} 
\usepackage[T1]{fontenc}    
\usepackage{hyperref}       
\usepackage{url, color}            
\usepackage{booktabs}       
\usepackage{nicefrac}       
\usepackage{microtype}      
\usepackage{subcaption}
\usepackage{graphicx}  
\title{Vector Quantization-Based Regularization for Autoencoders}
\author{\large \textbf{Hanwei Wu\textsuperscript{\rm 1, 2}} and
\textbf{Markus Flierl\textsuperscript{\rm 1}} \\
\{hanwei, flierl\}@kth.se \\
\textsuperscript{\rm 1}KTH Royal Institute of Technology, Stockholm, Sweden\\
\textsuperscript{\rm 2}Research Institutes of Sweden\\
Stockholm, Sweden}
\DeclareMathOperator*{\argmin}{arg\min}
\DeclareMathOperator*{\argmax}{arg\max}

\newtheorem{theorem}{Theorem}
\begin{document}
\maketitle
\begin{abstract}
Autoencoders and their variations provide unsupervised models for learning low-dimensional representations for downstream tasks. Without proper regularization, autoencoder models are susceptible to the overfitting problem and the so-called posterior collapse phenomenon. In this paper, we introduce a quantization-based regularizer in the bottleneck stage of autoencoder models to learn meaningful latent representations. We combine both perspectives of Vector Quantized-Variational AutoEncoders (VQ-VAE) and classical denoising regularization methods of neural networks. We interpret quantizers as regularizers that constrain latent representations while fostering a similarity-preserving mapping at the encoder. Before quantization, we impose noise on the latent codes and use a Bayesian estimator to optimize the quantizer-based representation. The introduced bottleneck Bayesian estimator outputs the posterior mean of the centroids to the decoder, and thus, is performing soft quantization of the noisy latent codes. We show that our proposed regularization method results in improved latent representations for both supervised learning and clustering downstream tasks when compared to autoencoders using other bottleneck structures.
\end{abstract}
\section{Introduction}
An important application of autoencoders and their variations is the use of learned latent representations for downstream tasks. In general, learning meaningful representations from data is difficult since the quality of learned representations is usually not measured by the objective function of the model. The reconstruction error criterion of vanilla autoencoders may result in the model memorizing the data. The variational autoencoders (VAE) \cite{kingma:14:iclr} improves the representation learning by enforcing stochastic bottleneck representations using the reparameterization trick. The VAE models further impose constraints on the latents by minimizing a Kullback–Leibler (KL) divergence between the prior and the approximate posterior of the latent distribution. However, the evidence lower bound (ELBO) training of VAE does not necessarily result in meaningful latent representations as the optimization cannot control the trade-off between the reconstruction error and the information transfer from the data to the latent representation \cite{pmlr-v80-alemi18a}. On the other hand, the VAE training is also susceptible to the so-called “posterior collapse” phenomenon where a structured latent representation is mostly ignored and the encoder maps the input data to the latent representation in a ``random'' fashion \cite{Lucas2019}. This is not favorable for downstream applications since the latent representation loses its similarity relation to the input data.

 Various regularization methods have been proposed to improve the latent representation learning for the VAE models. \cite{Higgins:17:iclr}\cite{burgess2018understanding} enforce stronger KL regularization on the latent representation in the bottleneck stage to constrain the transfer of information from data to the learned representation. Denoising methods \cite{8253482}\cite{NIPS2018_7692}\cite{bengio:AAAI:17} encourage the model to learn robust representations by artificially perturbing the training data. On the other hand, conventional regularization methods may not solve the posterior collapse problem. \cite{Lucas2019} empirically shows that posterior collapse is caused by the original marginal log-likelihood objective of the model rather than the evidence lower bound (ELBO). As a result, modifying the objective function ELBO of VAE as \cite{Higgins:17:iclr}\cite{burgess2018understanding} may have limited effects on preventing the posterior collapse. One potential solution is the vector-quantized variational autoencoder (VQ-VAE) \cite{oord:17:nips} model. Instead of regularizing the latent distribution, VQ-VAE provides a latent representation based on a finite number of centroids. Hence, the capability of the latent representation can be controlled by the number of used centroids which guarantees that a certain amount of information is preserved in the latent space.

In this paper, we combine the perspectives of VQ-VAE and noise-based approaches. We inject noise into the latent codes before the quantization in the bottleneck stage. We assume that the noisy observations are generated by a Gaussian mixture model where the means of the components is represented by the centroids of the quantizer. To determine the input of the autoencoder decoder, we use a Bayesian estimator and obtain the posterior mean of the centroids. In other words, we perform a soft quantization of the latent codes in contrast to a hard assignment as used in vanilla VQ-VAE. Hence, we refer to our framework as soft VQ-VAE. Since our focus is on using autoencoders to extract meaningful low-dimensional representations for other downstream tasks, we demonstrate that the latent representation extracted from our soft VQ-VAE models are effective in subsequent classification and clustering tasks in the experiments.
\section{Bottleneck Vector Quantizer}
\label{quantization}
\subsection{Vector Quantization in Autoencoders}
\label{subsec:vq-vae}
 Here we first give a general description of the autoencoder model with vector quantized bottleneck based on the VQ-VAE formulation. The notational conventions in this work are as follows: Boldface symbols such as $\mathbf{x}$ are used to denote random variables. Nonboldface symbols $x$ are used to denote sample values of those random variables.
 
 The bottleneck quantized autoencoder models consist of an encoder, a decoder, and a bottleneck quantizer. The encoder learns a deterministic mapping and outputs the \emph{latent code} $\mathbf{z_e} = g_{\text{enc}}(\mathbf{x})$, where $\mathbf{x} \in \mathcal{X} = \mathbb{R}^{D}$ denotes the input datapoint and $\mathbf{z_e} \in \mathbb{R}^{d}$. The latent code $\mathbf{z_e}$ can be seen as an efficient representation of the input $\mathbf{x}$, such that $d \ll D$. The latent code $\mathbf{z_e}$ is then fed into the bottleneck quantizer $Q(\cdot)$. The quantizer partitions the latent space into $K$ clusters characterized by the codebook $\mathcal{M}  = \{\mu^{(1)}, \cdots, \mu^{(K)}\}$. The latent code $\mathbf{z_e}$ is quantized to one of the $K$ codewords by the nearest neighbor search
 \begin{equation}
 \label{eq:assgn1}
 \mathbf{z_q} = Q(\mathbf{z_e}) = \mu^{(\mathbf{c})}, \ \text{where} \ \mathbf{c} = \argmin_k\|\mathbf{z_e}-\mu^{(k)}\|_2.
 \end{equation} 
 The output $\mathbf{z_q}$ of the quantizer is passed as input to the decoder. The decoder then reconstructs the input datapoint $\mathbf{x}$. 
\subsection{Bottleneck Vector Quantizer as a Latent Parameter Estimator}
In this section, we show that the embedded quantizer can be interpreted as a parameter estimator for the latent distribution with a discrete parameter space. In a vanilla VAE perspective, the encoder outputs the parameters of its latent distribution. The input of the VAE decoder is sampled from the latent distribution that is parameterized by the output of the VAE encoder. For bottleneck quantized autoencoders, the embedded quantizer creates a discrete parameter space for the model posterior. The nearest neighbor quantization effectively makes the autoencoder a generative model of Gaussian mixtures with a finite number of components, where the components are characterized by the codewords of the quantizer \cite{Henter2018}. In contrast, the vanilla VAE is equivalent to a mixture of an infinite number of Gaussians as the latent parameter space is continuous. Furthermore, the variational inference model of the bottleneck quantized autoencoder can be expressed as
	\begin{align}
	q(z|x) &= \sum_{k = 1}^Kq\left(z|z_q = \mu^{(k)}\right)q\left(z_q = \mu^{(k)}|x\right) \\
	&= q\left(z|z_q = Q(g_{\text{enc}}(x))\right)\delta\left(z_q = Q(g_{\text{enc}}(x))\right),
	\end{align}
	where $z$ is the latent variable and $\delta(\cdot)$ is the indicator function. 

As a result, the decoder input $z_q$ can be seen as the estimated parameters of the latent distribution with the discrete parameter space that is characterized by the codebook of the quantizer. That is, the decoder of the bottleneck quantized autoencoders takes the estimated parameter of the latent distribution and recover the parameters of the data generating distribution of the observed variables $\mathbf{x}$. No sampling of $\mathbf{z}$ from the latent distribution is needed during the training of vector-quantized autoencoder models.
\subsection{Vector Quantizer as a Regularizer}
We showcase that the added quantizer between encoder and decoder acts also as a regularizer on the latent codes that fosters similarity-preserving mappings at the encoder for Gaussian observation models. We use visual examples to show that the embedded bottleneck quantizer can enforce the encoder output to share a constrained coding space such that learned latent representations preserve the similarity relations of the data space. We argue that this is one of the reasons that bottleneck quantized autoencoders can learn meaningful representations. 

Assume that we have a decoder with infinite capacity. That is, the decoder is so expressive that it can produce a precise reconstruction of the input of the model without any constraints on the latent codes. As a result, the encoder can map the input to the latent codes in an arbitrary fashion while keeping a low reconstruction error (See Fig. \ref{fig: q1}).

With the quantizer inserted between the encoder and decoder, the encoder can only map the input to a finite number of representations in the latent space. For example, in Fig. \ref{fig: q2}, we insert a codebook with two codewords. If we keep the encoder mapping the same as Fig. \ref{fig: q1}, then, both blue and purple nodes in the latent space will be represented by the blue node in the discrete latent space due to the nearest neighbor search. In this case, the optimal reconstruction of the blue and purple nodes at the input will be the green node at the output. This is obviously not the optimal encoder mapping with respect to the reconstruction error. Instead, the more efficient mapping of the encoder is to map similar data points to neighboring points in the latent space (See Fig. \ref{fig:sub3}).
\begin{figure}[!ht]
\captionsetup[subfigure]{justification=centering}
\centering
	\begin{subfigure}{.45\textwidth}
	\centering
	\includegraphics[width=.8\linewidth]{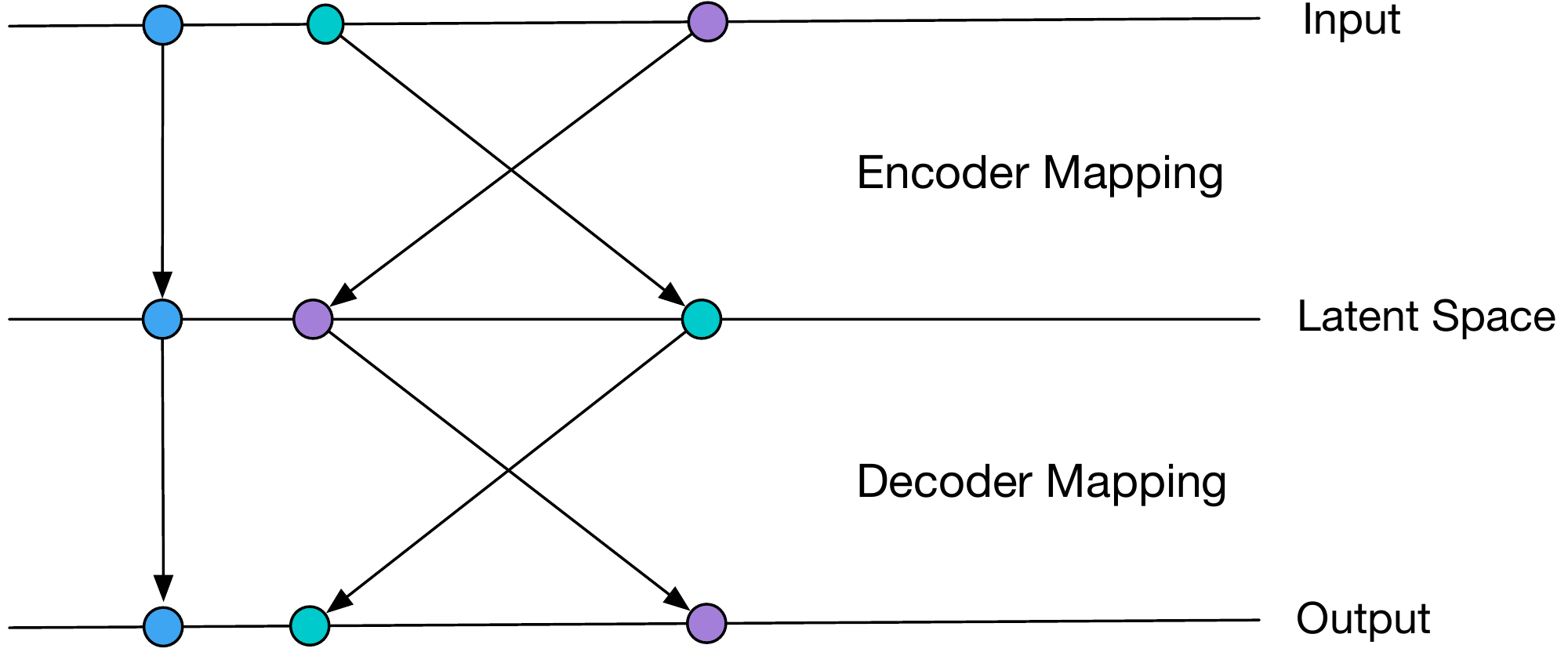}
	\caption{Vanilla autoencoders}
		\label{fig: q1}
	\end{subfigure}
\begin{subfigure}{.45\textwidth}
\centering
	\includegraphics[width=.8\linewidth]{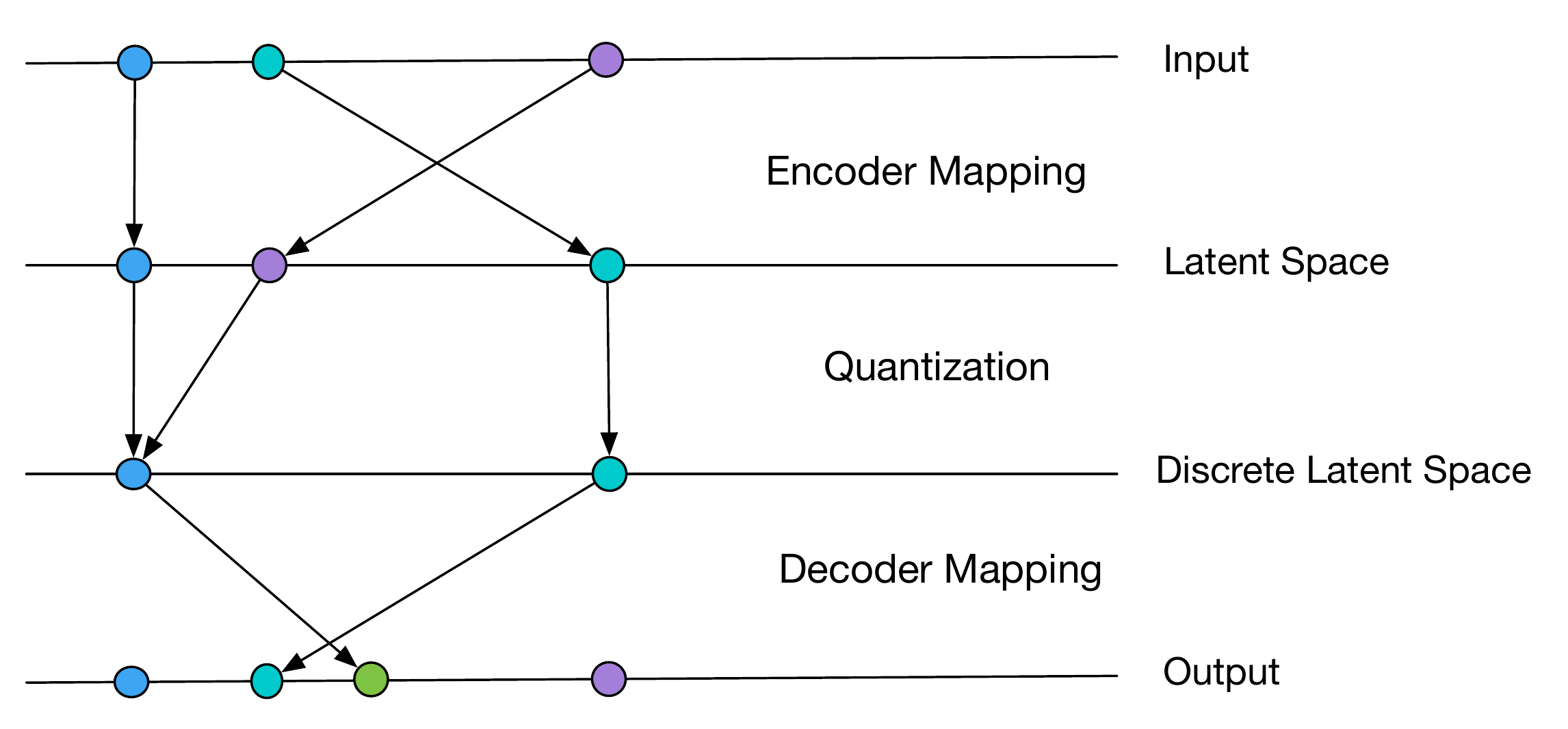}
	\caption{Autoencoders with quantized bottleneck}
		\label{fig: q2}
\end{subfigure}
\begin{subfigure}{.45\textwidth}
\centering
\includegraphics[width=.8\linewidth]{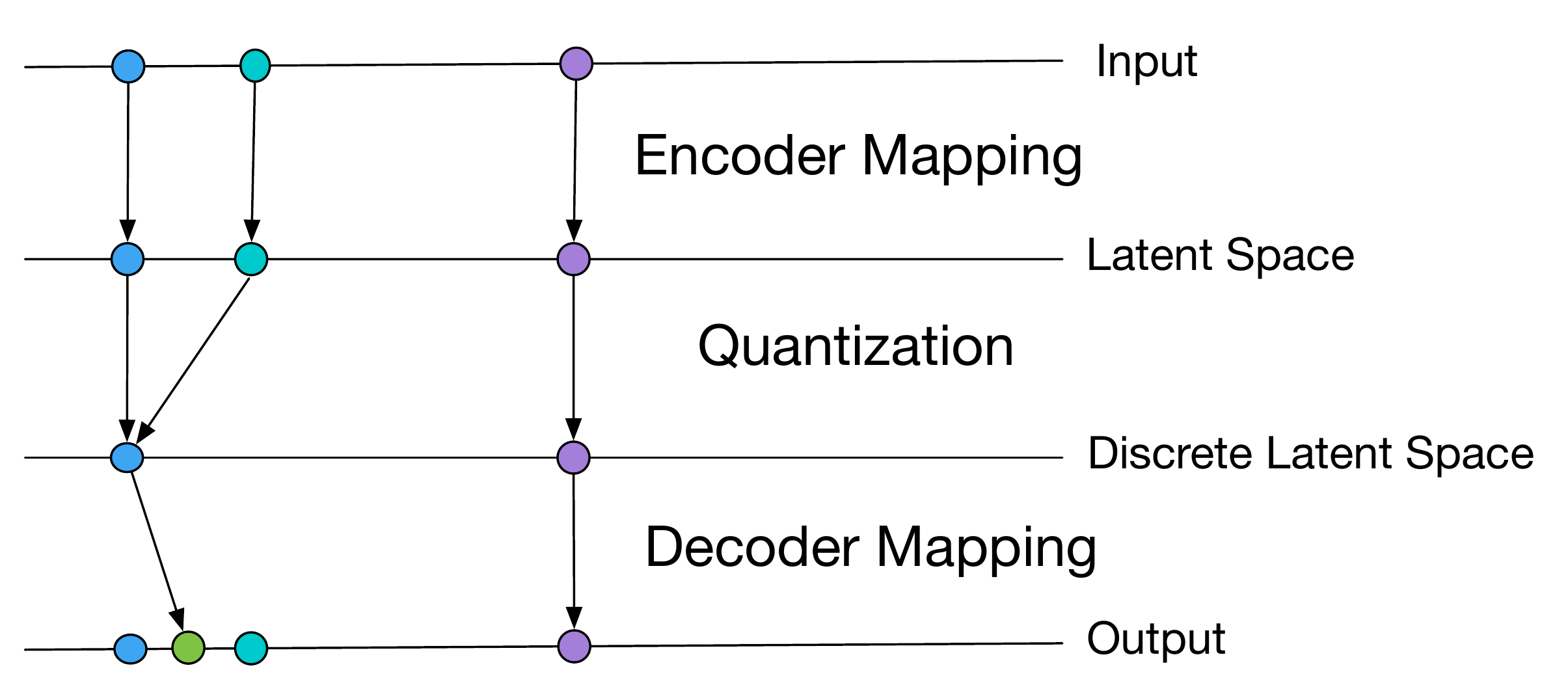}
\caption{The quantizer enforces a\\ similarity-preserving mapping at the encoder}
\label{fig:sub3}
\end{subfigure}
\caption{The quantizer behaves as a regularizer that encourages a similarity-preserving mapping} at the encoder.
\label{fig:test}
\end{figure}

However, we can also observe that the bottleneck quantized autoencoders inevitably hurts the reconstruction due to the limited choices of discrete latent representations. That is, the number of possible reconstructions produced by a decoder is limited by the size of the codebook. This is insufficient for many datasets who have a large number of classes. In our proposed soft VQ-VAE, it increases the expressiveness of the latent representations by using a Gaussian mixture model and the decoder input is a convex combination of the codewords.

\section{Soft VQ-VAE}
\label{noise_injection}
\subsection{Noisy Latent Codes}
Injecting noise on the input training data is a common technique for learning robust representations \cite{bengio:AAAI:17}. In our paper, we extend this practice by adding noise to the latent space such  that the models are exposed to new data samples. We note that this practice is also applied in \cite{DBLP:journals/corr/abs-1811-07557} and \cite{pmlr-v32-rezende14} to improve the generalization ability of their models.

We propose to add white noise $\bm{\epsilon}$ with zero mean and finite variance on the encoder output $\mathbf{z_e'} = \mathbf{z_e} + \bm{\epsilon}$, where $\bm{\epsilon} \in \mathbb{R}^d$.  We assume that the added noise variance $\sigma_\epsilon$ is unknown to the model. Instead, we view the noisy latent code is generated from a mixture model with $K$ components
\begin{equation}
\label{eq:gmm_generative}
p(z_e') = \sum_{k = 1}^K p\left(z_q = \mu^{(k)}\right)p\left(z_e'|z_q = \mu^{(k)}\right), 
\end{equation}
where $z_q \in \mathcal{M}$.

We let the conditional probability function of $\mathbf{z_e'}$ given one of the codewords $\mu^{(k)}$ to be a multivariate Gaussian distribution $\mathcal{N}\left(\mu^{(k)}, I^{(k)}\right)$ 
\begin{equation}
\label{eq:componentGaussian}
\resizebox{.97\hsize}{!}
{$p\left(z_e'|\mu^{(k)}\right) = \frac{\exp\left(\left(-\frac{1}{2}(z_e'-\mu^{(k)}\right)^T I^{(k)-1}\left(z_e'-\mu^{(k)}\right)\right)}{\sqrt{(2\pi)^d|I^{(k)}|}}$},
\end{equation}
where the $k$-th codeword $\mu^{(k)}$ is regarded as the mean of the Gaussian distribution of the $k$-th component, $I^{(k)} = \sigma_k^2I$ and $\sigma_k$ is the standard deviation of the $k$-th component.
\subsection{Bayesian Estimator}
 We add a Bayesian estimator after the noisy latent codes in the bottleneck stage of the autoencoder. The aim is to estimate the parameters of the latent distribution from noisy observations.
 The Bayesian estimator is optimal with respect to the mean square error (MSE)  criterion and is defined as the mean of the posterior distribution,
\begin{equation}
\label{eq: estimator}
\hat{z}_q = \mathbb{E}[\mathbf{z_q}|z_e'] = \sum_{k = 1}^{K}\mu^{(k)}p\left(\mu^{(k)}|z_e'\right).
\end{equation}
Using Bayes' rule, we express the conditional probability $p\left(\mu^{(k)}|z_e'\right)$  as
\begin{align}
p\left(\mu^{(k)}|z_e'\right) = \frac{p\left(\mu^{(k)}\right)p\left(z_e'|\mu^{(k)}\right)}{p(z_e')},
\end{align}
where we assume an uninformative prior for the codewords $p\left(\mu^{(k)}\right) = \frac{1}{K}$ as there is no preference for single codeword. The conditional probability $p\left(z_e'|\mu^{(k)}\right)$ is given in (\ref{eq:componentGaussian}) and the marginal distribution of the noisy observation is given by marginalizing out the finite codebook in (\ref{eq:gmm_generative}).

Compared to the hard assignment of the VQ-VAE, we can see that we are equivalently performing a soft quantization as the noisy latent code is assigned to a codeword with probability $p\left(\mu^{(k)}|z_e'\right)$. The output of the estimator is a convex combination of all the codewords in the codebook. The weight of each codeword is determined similar to a radial basis function kernel where the value is inversely proportional to the L$2$ distance between $\mathbf{z_e'}$ and a codeword with component variance as the smoothing factor.
Fig. \ref{fig:estimator} shows the described soft VQ-VAE.
\begin{figure}[!ht]
	\centering
	\includegraphics[width=.45\textwidth]{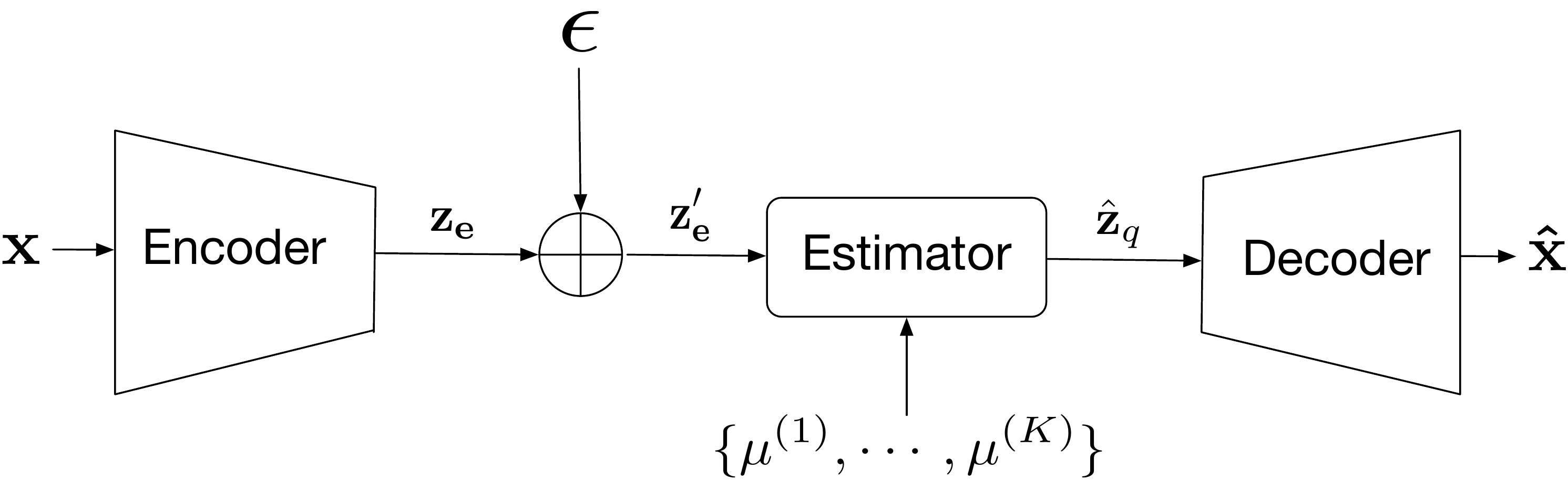}
	\caption{Description of the soft VQ-VAE.}
	\label{fig:estimator}
\end{figure}
\subsection{Optimal Estimator} 
In this section, we show that our added Bayesian estimator is optimal with respect to the model evidence of the bottleneck quantized autoencoders with noisy latent codes. The maximum likelihood principle of generative models chooses the model parameters that maximize the likelihood of the training data \cite{DBLP:journals/corr/Goodfellow17}. Similarly, we can decompose the marginal log-likelihood of the model distribution as the model ELBO plus the KL divergence between the variational distribution and the model posterior \cite{8588399},
\begin{equation}
	\log p(x) = \mathbb{E}_{q(z)} \left[\log\left(\frac{p(x, z)}{q(z)}\right)\right] + \text{KL}(q(z)\|p(z|x)),
\end{equation}
where $p(z|x)$ is the model posterior and $q(z)$ is the variational latent distribution that regularizes the model posterior. The maximization of the model ELBO can be seen as searching for the optimal latent distribution $q$ within a variational family $\mathcal{Q}$ that approximates the true model posterior $p(z|x)$. Given a uniform distribution $\hat{p}(x)$ over the training dataset, we can obtain the optimal estimation of the latent distribution: 
\begin{align}
q^* &= \mathbb{E}_{\hat{p}({x})}\argmax_{q \in \mathcal{Q}}\left[\mathbb{E}_{q(z)} \left[ \log\left(\frac{p(x,z)}{q(z)}\right)\right]\right] \\
\label{eq:elboalter}
&= \mathbb{E}_{\hat{p}(x)}[\log p(x)] - \mathbb{E}_{\hat{p}(x)}\argmin_{q \in \mathcal{Q}}[\text{KL}(q(z) \|p(z|x))]\\
\label{eq:klgoal}
&=\mathbb{E}_{\hat{p}(x)}\argmin_{q \in \mathcal{Q}}[\text{KL}(q(z) \|p(z|x))],
\end{align}
Since the first term of (\ref{eq:elboalter}) is irrelevant with respect to the approximated latent distribution, the maximization of the model ELBO becomes equivalent to finding the latent distribution that minimizes the KL divergence to the model posterior distribution in (\ref{eq:klgoal}).

For bottleneck quantized autoencoders, the embedded quantizer enforces the model posterior $p(z|x)$ to be the unimodal distribution centered on one of the codewords $\mu \in \mathcal{M}$. In our noisy model, we perturb the encoder output $\mathbf{z_e}$ by random noise. We assume that the noise variance is unknown to the model and the parameter cannot be determined by performing a nearest neighbor search on the noisy bottleneck representation $\mathbf{z_e'}$. Instead, the introduced Bayeisan estimator (\ref{eq: estimator}) outputs a convex combination of the codewords. In the following Theorem, we show that our proposed Bayesian estimator outputs the parameters of the optimal latent distribution for the quantized bottleneck autoencoder models under the condition that the latent distribution belongs to the Gaussian family. 
\begin{theorem}
	\label{propostion1}
	Let $\mathcal{Q}$ be the set of Gaussian distributions with associated parameter space $\Omega$. Based on the described noisy model, for one datapoint, the estimator $f$: $\mathcal{X} \rightarrow \Omega$ that outputs the parameters of the optimal $q^* \in \mathcal{Q}$ is given by
	\begin{equation}
	\label{eq: optiesti}
	\hat{z}_q = f(x) = \sum_{k = 1}^{K}\mu^{(k)}p\left(\mu^{(k)}|z_e'\right).
	\end{equation}
\end{theorem}
\begin{proof}
	For the noisy setting, the expectation of the KL divergence between the model posterior $p(z|x)$ and the approximated $q$ is taken with respect to the empirical training distribution $\hat{p}(x)$ and the noise distribution $\hat{p}(\epsilon)$
	\begin{equation}
	\label{eq:setp1}
	\mathbb{E}_{\hat{p}(x)}\mathbb{E}_{\hat{p}(\epsilon)}[\text{KL}(q(z) \|p(z|x))].
	\end{equation}
	Since the encoder does not have activation functions in the output layer, we assume that the encoder neural network is a deterministic injective function over the empirical training set such that $\hat{p}(x) = \hat{p}(z_e)$. Also, the injected noise is independent of $z_e$, we can express the probability distribution of the training data and noise as the following chain of equalities: \\
	\begin{equation}
	\hat{p}(x)\hat{p}(\epsilon) =\hat{p}(z_e)\hat{p}(\epsilon) =  \hat{p}(z_e, \epsilon) = \hat{p}(z_e, z_e + \epsilon) =  \hat{p}(z_e, z_e')
	\end{equation}
	The joint probability $\hat{p}(z_e, z_e')$ can be further decomposed as 
	\begin{align}
	\hat{p}(z_e, z_e') &= \hat{p}(z_e)\hat{p}(z_e'|z_e)\\
	&= \hat{p}(z_e)\sum_{k =1}^{K}\hat{p}\left(z_q = \mu^{(k)}|z_e\right)\hat{p}\left(z_e'|z_q = \mu^{(k)}\right) \\
	\label{eq:setp2}
	& = \hat{p}(z_e)\frac{1}{K}\sum_{k = 1}^{K}\hat{p}\left(z_e'|\mu^{(k)}\right).
	\end{align}
	where (\ref{eq:setp2}) follows from that $z_e$ is considered as unobservable in the model (see Fig. \ref{fig:markov_soft}), and thus provides no information about $\mu^{(k)}$ such that the conditional probability of $\mu^{(k)}$ given $z_e$ is equal to the prior of the codewords $\frac{1}{K}$.
	\begin{figure}[!ht]
		\centering
		\includegraphics[width = .34\linewidth]{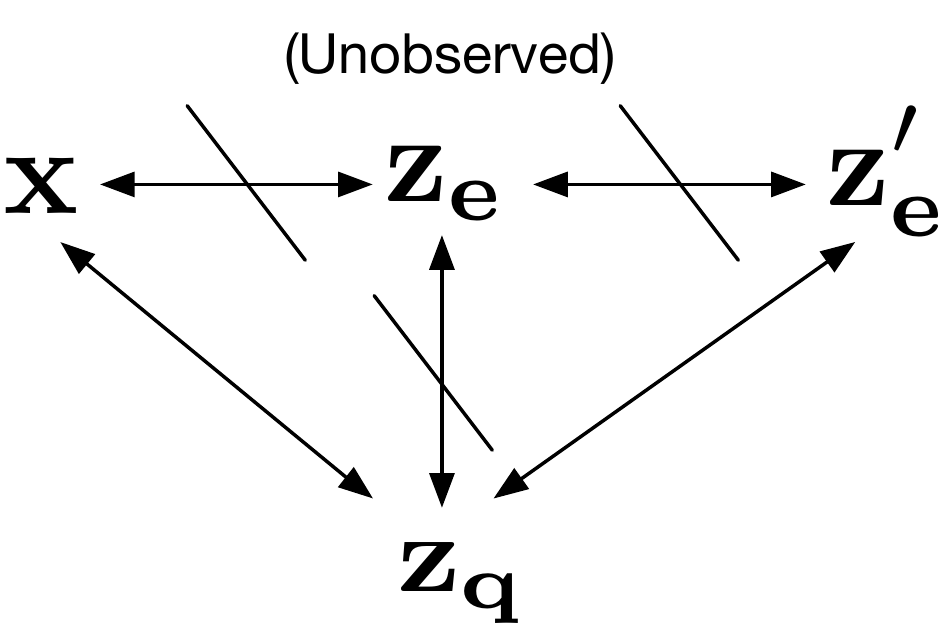}
		\caption{The relation of variables in the soft VQ-VAE model. The symbol $\nleftrightarrow$ is used to indicate that we cannot directly use paths that connected to the unobserved $\mathbf{z_e}$ for probabilistic inference.}
		\label{fig:markov_soft}
	\end{figure}
    
    Combining the model posterior of bottleneck quantized autoencoders and the above derivations, we can reexpress (\ref{eq:setp1}) as
	\begin{align}
	&\mathbb{E}_{\hat{p}(z_e, z_e')}[\text{KL}(q(z) \|p(z|z_q)]\\
	=&\mathbb{E}_{\hat{p}(z_e)}\frac{1}{K}\sum_{k = 1}^{K}\mathbb{E}_{\hat{p}\left(z_e'|\mu^{(k)}\right)}\left[\text{KL}(q(z)\|p\left(z|\mu^{(k)}\right)\right],
	\end{align}
	where the true model posterior has $p(z|x) = p(z|z_q)$.
	
	Therefore, for each datapoint $x$, the optimization problem with respect to the latent distribution $q$ (\ref{eq:klgoal}) for the noisy setting becomes
	\begin{align}
	\label{eq:berg_obj}
	&\min_{q \in \mathcal{Q}}\frac{1}{K}\sum_{k = 1}^K\mathbb{E}_{\hat{p}\left(z_e'|\mu^{(k)}\right)}\left[\text{KL}\left(q(z)\|p\left(z|\mu^{(k)}\right)\right)\right]\\
    =& \min_{q \in \mathcal{Q}}\frac{1}{K}\sum_{k = 1}^K \hat{p}\left(z_e'|\mu^{(k)}\right)\text{KL}\left(q(z)\|p\left(z|\mu^{(k)}\right)\right).
	\end{align}
	
	Note that the KL divergence between two exponential family distributions can be represented by the Bregman divergence $d_A(\cdot)$ between the corresponding natural parameters $\eta'$ and $\eta$ as 
	\begin{align}
	\text{KL}(p_{\eta'}\|p_{\eta}) &=  d_A(\eta, \eta')\\
	&=-A(\eta')+A(\eta)-\nabla A(\eta)^T(\eta'-\eta) ,
	\end{align}
	where $A(\cdot)$ is the log-partition function for the exponential family distribution. Furthermore, it has been shown that the minimizer of the expected Bregman divergence from a random vector is its mean vector \cite{bregman}. Therefore, we formulate (\ref{eq:berg_obj}) as a convex combination of the KL divergence
	\begin{align}
	\label{eq: convex_obj}
	&\argmin_{q \in \mathcal{Q}} \sum_{k = 1}^ K\omega_k \text{KL}\left(q(z)\| p\left(z|\mu^{(k)}\right)\right) \\
	=&\argmin_{\eta}\sum_{k = 1}^K \omega_k d_A(\eta^{(k)}, \eta),
	\end{align}
	where $\omega_k = \frac{1}{VK}\hat{p}\left(z_e'|\mu^{(k)}\right)$. $V = \sum_{k = 1}^K\hat{p}\left(z_e'|\mu^{(k)}\right)$ is the introduced normalization constant and the optimal solution of (\ref{eq:berg_obj}) is not affected. In addition, due to the normalization, $\omega_k$ becomes $p\left(\mu^{(k)}|z_e'\right)$. Then, the minimizer of (\ref{eq: convex_obj}) is given by the mean of $\eta^{(k)}$
	\begin{equation}
	\label{eq: minimizer}
	\eta =\sum_{k = 1}^{K}p\left(\mu^{(k)}|z_e'\right)\eta^{(k)}.
	\end{equation}
	
	The natural parameters for the multivariate Gaussian distribution with known covariance matrix is $\Sigma^{-1}\mu$. Since the $p(z|x)$ is the model posterior of the noiseless bottleneck quantized autoencoders, the covariance matrix is assumed to be the identity matrix for all components $\Sigma = I$. Therefore, we can recover the Bayesian estimator (\ref{eq: optiesti}) by substituting $\eta^{(k)}$ with $\mu^{(k)}$ in (\ref{eq: minimizer}), and the proof is complete.
\end{proof}
\section{Related Work}
\label{Background}
For extended work on VQ-VAE, \cite{roy:18} uses the Expectation Maximization algorithm in the bottleneck stage to train the VQ-VAE and to achieve improved image generation results. However, the stability of the proposed algorithm may require to collect a large number of samples in the latent space. \cite{Henter2018} gives a probabilistic interpretation of the VQ-VAE and recovers its objective function using the variational inference principle combined with implicit assumptions made by the vanilla VQ-VAE model.

Several works have studied the end-to-end discrete representation learning model with different incorporated structures in the bottleneck stages. \cite{theis:17:iclr} and \cite{Balle:17:iclr}  introduce scalar quantization in the latent space and optimize jointly the entire model for rate-distortion performance over a database of training images. \cite{agustsson:17:nips} proposes a compression model by performing vector quantization on the network activations. The model uses a continuous relaxation of vector quantization which is annealed over time to obtain a hard clustering. In \cite{agustsson:17:nips}, the softmax function is used to give a soft assignment to the codewords where a single smoothing factor is used as an annealing factor. In our model, we learn different smoothing factors for each component. \cite{sonderbypoole2017} introduces a continuous relaxation training of discrete latent-variable models which can flexibly capture both continuous and discrete aspects of natural data.

Various techniques for regularizing the autoencoders have been proposed recently. \cite{DBLP:journals/corr/abs-1807-07543} proposes an adversarial regularizer which encourages interpolation in the outputs and also improves the learned representation. \cite{NIPS2018_7692} interprets the VAEs as a amortized inference algorithm and proposed a procedure to constrain the expressiveness of the encoder. In addition, there is a increasing popularity of using information-theoretic principles to improve autoencoders. \cite{alemi:17:iclr}\cite{pmlr-v80-alemi18a} use the information bottleneck principle \cite{tishby:15:itw} to recover the objective of $\beta$-VAE and show that the KL divergence term in ELBO is an upper bound on the information rate between input and prior. \cite{8253482} is also inspired by the information bottleneck principle and  introduces the information dropout method to penalize the transfer of information from data to the latents. \cite{DBLP:journals/corr/abs-1811-07557} proposes to use encoder-decoder structures and inject noises to the bottleneck stage to simulate binary symmetric channels (BSC). By jointly optimizing the encoding and decoding processes, the authors show that the trained model not only can produce codes that have better performance for the joint source-channel coding problem but also that the noisy latents facilitate robust representation learning.

We also note that the practice of using a convex combination of codewords is similar to the attention mechanism \cite{NIPS2017_7181}. The attention mechanism is introduced to solve the gradient vanishing problem that models fail to learn the long-term dependencies of time series data. It can be viewed as a feed-forward layer that takes the hidden state of the at each time step as input and outputs the so-called context vector as the representation which is a weighted combination of the input hidden state vectors. 
\section{Experimental Results}
\subsection{Model Implementation}
We test our proposed model on datasets MNIST, SVHN and CIFAR-10. All the tested autoencoder models share the same encoder-decoder setting. For the models tested on the SVHN and CIFAR-10, we use convolutional neural networks (CNN) to construct the encoder and decoder. For the MNIST, we use multilayer perceptron (MLP) networks to construct encoder and decoder. All decoders follow a structure that is symmetric to the encoder.

The differences among the compared models are only in the bottleneck operation. The bottleneck operation takes the encoder output as its input, and its output is fed into the decoder. For VAE and information dropout models, the bottleneck input is two separate encoder output layers of $d$ units respectively. One layer learns the mean of the Gaussian distribution and the other layer learns the log variance. The reparameterization trick or the information dropout technique is applied to generate samples for the latent distribution. For the VQ-VAE, the bottleneck performs a nearest neighbor search on the encoder output. Then, the quantized codeword is fed into the decoder. For the soft VQ-VAE, the bottleneck input is also two separate encoder output layers. One layer of size $d$ outputs the noiseless vector $z_e$. Another layer with size $K$ outputs the log variance of components. The noise injection is performed only on $z_e$ and the estimator uses the noisy samples and the variances of components for estimation. The baseline autoencoder directly feeds the encoder output to the decoder. 

 The soft VQ-VAE models are trained in a similar fashion as VQ-VAE. Specifically, the loss function for the soft VQ-VAE mode is
\begin{equation}
	\label{eq:loss1}
	\begin{split}
	L = -\log p(x|\hat{z}_q)+ \|\text{sg}\left(z_e'\right)-\hat{z}_q\|_2^2 + \beta\|z_e'- \text{sg}(\hat{z}_q)\|_2^2.
	\end{split}
\end{equation}
where $\text{sg}(\cdot)$ denotes the stop gradient operator and $\beta$ is a hyperparameter to encourage the encoder to commit to a codeword. The stop gradient operator is used to solve the vanishing gradient problem of discrete variables by separating the gradient update of encoder-decoder and the codebook.
The $\text{sg}(\cdot)$ outputs its input when it is in the
forward pass, and outputs zero when computing gradients in the training process.
Specifically, the decoder input is expressed as $\hat{z}_q = z_e + \text{sg}(\hat{z}_q-z_e)$ such that the gradients are copied from the decoder input to the encoder output.
\subsection{Training Setup}
  For the models tested on the CIFAR-10 and SVHN datasets, the encoder consists of $4$ convolutional layers with stride $2$ and filter size $3 \times 3$. The number of channels is doubled for each encoder layer. The number of channels of the first layer is set to be $64$. The decoder follows a symmetric structure of the encoder. For MINST dataset, we use multilayer perceptron networks (MLP) to construct the autoencoder. The dimensions of dense layers of the encoder and decoder are $D$-$500$-$500$-$2000$-$d$ and $d$-$2000$-$500$-$500$-$D$ respectively, where $d$ is the dimension of the learned latents and $D$ is the dimension of the input datapoints. All the layers use rectified linear units (ReLU) as activation functions. 
  
  We use the Glorot uniform initializer \cite{pmlr-v9-glorot10a} for the weights of encoder-decoder networks. The codebook is initialized by the uniform unit scaling. All models are trained using Adam optimizer \cite{kingma:15:iclr} with learning rate 3e-4 and evaluate the performance after $40000$ iterations with batch size $32$. Early stopping at $10000$ iterations is applied by soft VQ-VAE on SVHN and CIFAR-10 datasets.  
\subsection{Visualization of Latent Representation}
In this section, we use t-SNE \cite{tSNE} to visualize the latent representations that have been learned by different autoencoder models, and examine their similarity-preserving mapping ability.
First, we train autoencoders with a $60$-dimensional bottleneck on the MNIST dataset. After the training, we feed the test data into the trained encoder to obtain the latent representation of the input data. The $60$-dimensional latent representations are projected into the two-dimensional space using the t-SNE technique. In Fig. \ref{fig:latent}, we plot the two-dimensional projection of the bottleneck representation $z_e$ of the trained models with different bottleneck structures. All autoencoder models are trained to have similar reconstruction quality. It is shown that the latent representation of the soft VQ-VAE preserves the similarity relations of the input data better than the other models.
\begin{figure}
	\centering
	\begin{subfigure}[b]{.3\textwidth}
		\centering
		\includegraphics[width=\textwidth]{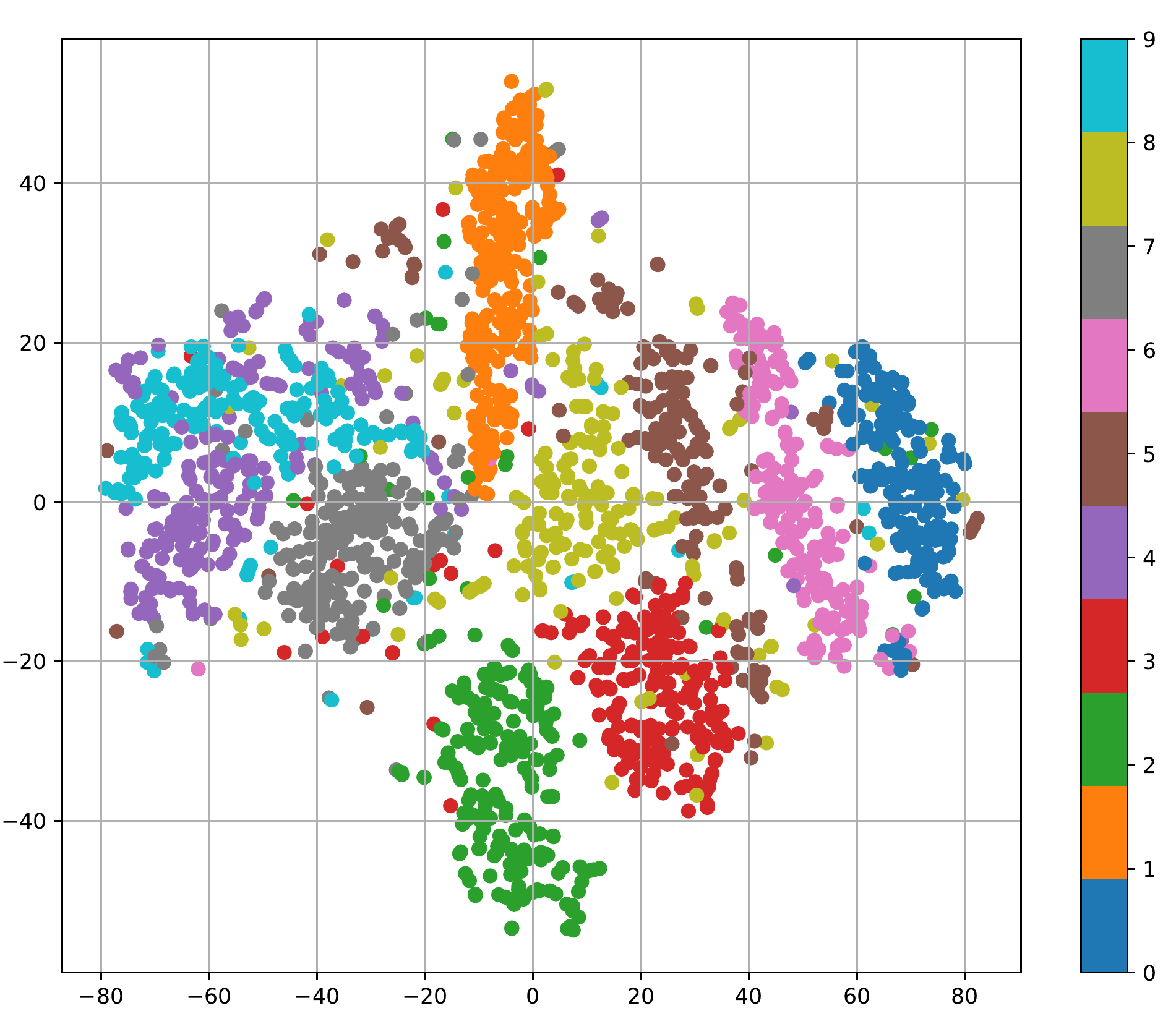}
		\caption[]{Autoencoder: $\mathbf{z_e}$.} 
		\label{fig:latents1}
	\end{subfigure}
	\begin{subfigure}[b]{.3\textwidth}
		\centering
		\includegraphics[width=\textwidth]{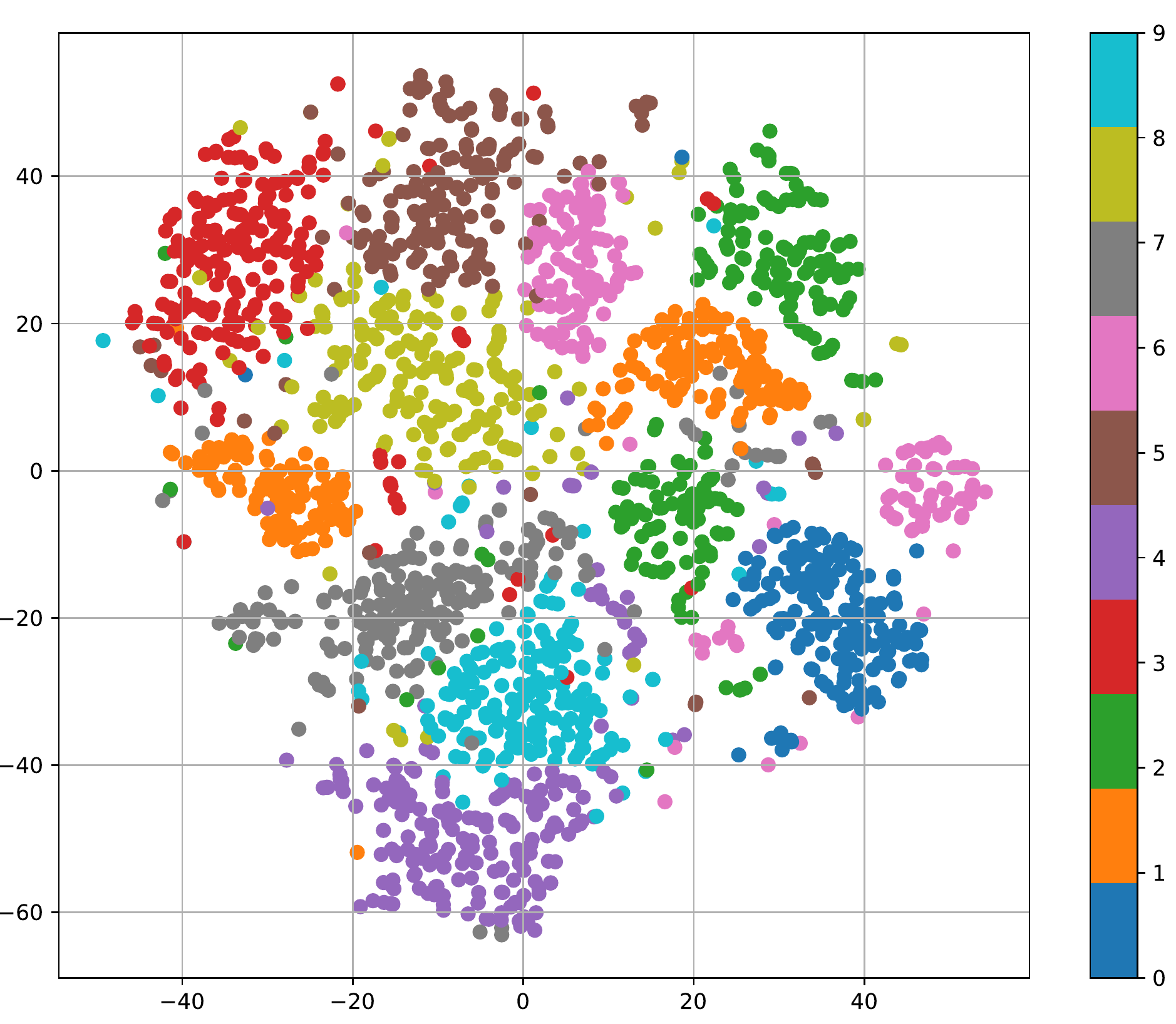} 
		\caption[]{VAE: $\mathbf{z_e}$.}%
		\label{fig:latents2}
	\end{subfigure}
	\begin{subfigure}[b]{.3\textwidth}
	\centering
	\includegraphics[width=\textwidth]{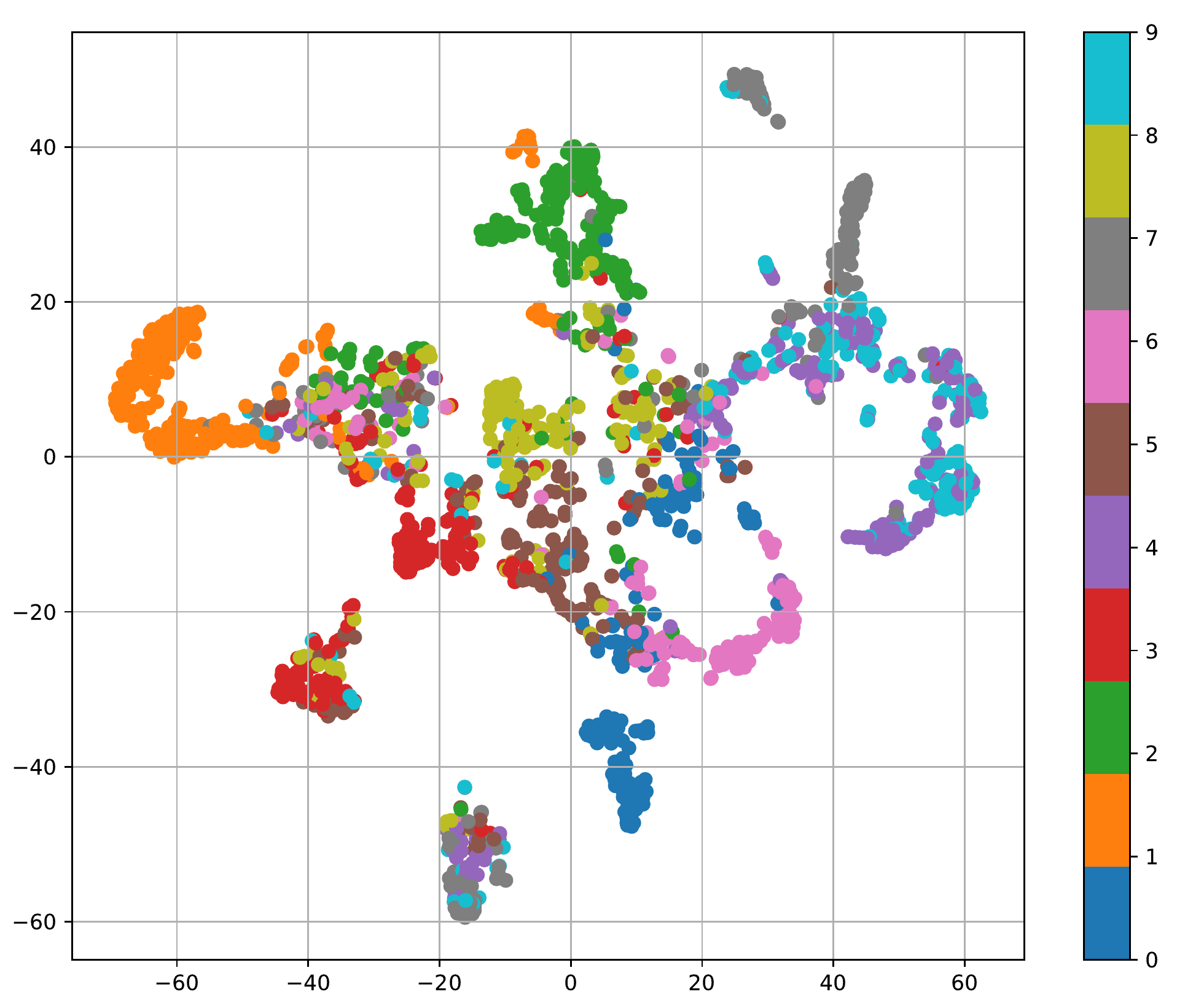}
	\caption[]{VQ-VAE: $\mathbf{z_e}$.} 
	\label{fig:latents3}
	\end{subfigure}
	\begin{subfigure}[b]{.3\textwidth}
		\centering
		\includegraphics[width=\textwidth]{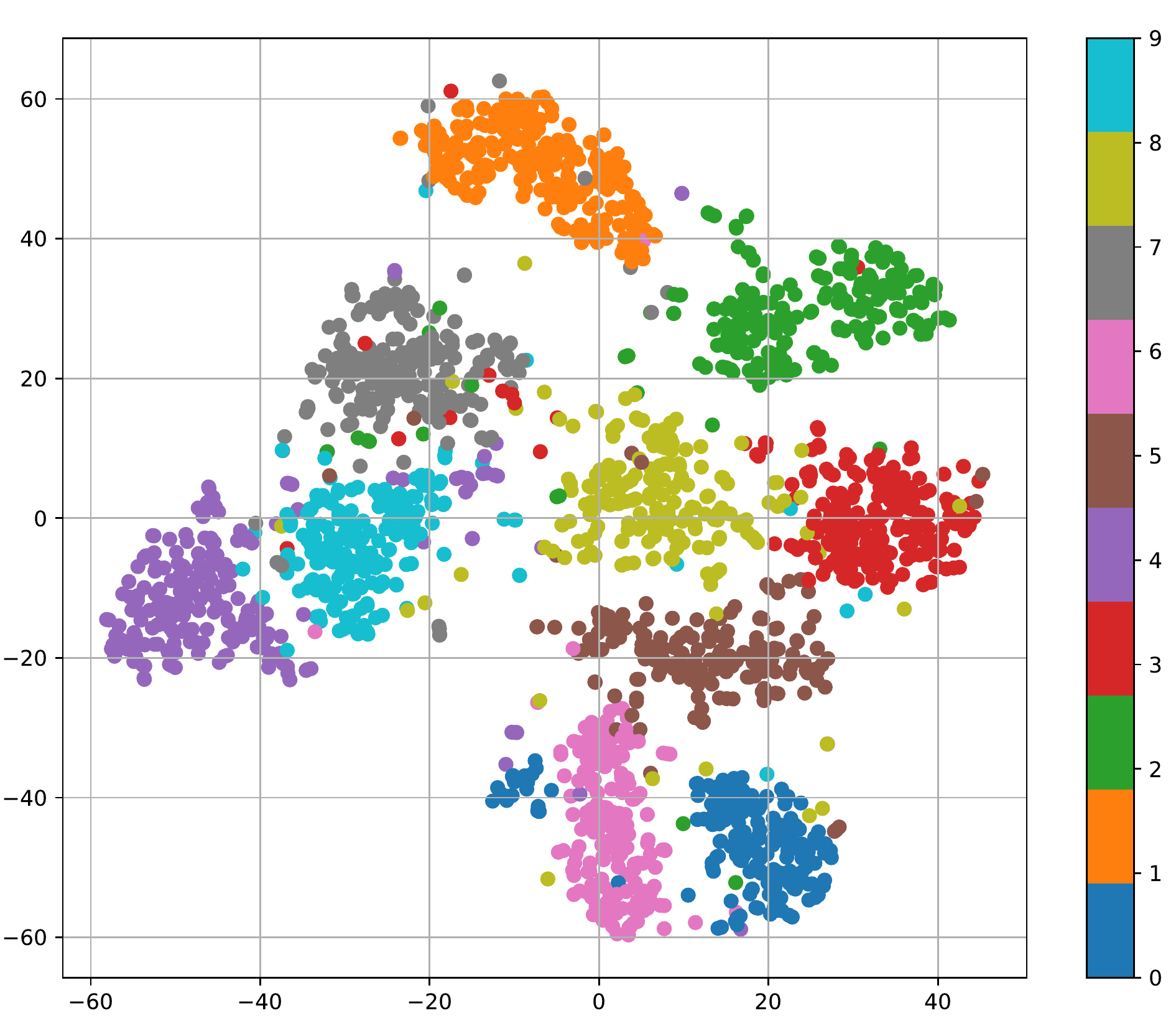} 
		\caption[]{soft VQ-VAE: $\mathbf{z_e}$.}%
	\label{fig:latents4}	
\end{subfigure}
\caption{Two-dimensional learned representations \\of MNIST. Each color indicates one digit class.}
\label{fig:latent}
\end{figure}
 \subsection{Representation Learning Tasks}
 \label{sb: downstreamtasks}
 We test our learned latent representation $z_e$ on K-means clustering and single-layer classification tasks as \cite{DBLP:journals/corr/abs-1807-07543}. The justification of these two tests is that if the learned latents can recover the hidden structure of the raw data, they should become more amiable to the simple classification and clustering tasks. We first train models using the training set. Then we use the trained model to project the test set on their latent representations and use them for downstreaming tasks. 
 
 For the K-means clustering, we use $100$ random initializations and select the best result. The clustering accuracy is determined by Hungarian algorithm \cite{xie:16:icml}, which is a one-to-one optimal linear assignment (LS) matching algorithm between the predicted labels and the true labels. We also test the clustering performance using the normalized mutual information (NMI) metric for the MNIST dataset \cite{Somvae} \cite{DBLP:journals/corr/abs-1801-07648}. The NMI is defined as NMI $(\mathbf{y}, \hat{\mathbf{y}}) = \frac{2I(\mathbf{y}, \hat{\mathbf{y}})}{H(\mathbf{y}) + H(\hat{\mathbf{y}})}$, where $\mathbf{y}$ and $\hat{\mathbf{y}}$ denote the true labels and the predicted labels, respectively. $I(\mathbf{y}, \hat{\mathbf{y}})$ is the mutual information between the predicted labels and the true label. $H(\cdot)$ is the entropy. For the classification tasks, we use a fully connected layer with a softmax function on the output as our classifier. The single-layer classifier is trained on the latent representation of the training set and is independent of the autoencoders' training.
\begin{table}[!ht]
	\caption{Accuracy of downstream tasks of MNIST.}
	\centering
	\resizebox{.95\columnwidth}{!}{
	\begin{tabular}{llll}
		\toprule
	    &\multicolumn{3}{c}{MNIST, $d = 64$} \\
		Model&Clustering&Clustering (NMI)&Classification\\
		\midrule
		Raw Data&\hfil$55.17$&\hfil$0.5008$&\hfil$92.44$\\
		Baseline Autoencoder&\hfil$52.61$&\hfil$0.5301$&\hfil$91.91$\\
		VAE&\hfil$56.44$&\hfil$0.5600$&\hfil$89.10$\\
		$\beta$-VAE ($\beta$ = 20)&\hfil$73.81$&\hfil$0.5760$&\hfil$91.10$\\
		Information dropout&\hfil$58.52$&\hfil$0.4979$&\hfil$91.11$\\
		VQ-VAE (K = 128) &\hfil$51.48$&\hfil$0.3541$ &\hfil$81.62$  \\
		Soft VQ-VAE (K=128)&\hfil$\textbf{77.64}$&\hfil$\textbf{0.7188}$&\hfil\textbf{93.54}\\
		\bottomrule
	\end{tabular}}
	\label{MNIST_table}
\end{table}
\begin{table}[!ht]
	\caption{Accuracy of downstream tasks of SVHN and CIFAR-10.}
	\centering
	\resizebox{.95\columnwidth}{!}{
	\begin{tabular}{llllll}
		\toprule
	    & \multicolumn{2}{c}{SVHN, $d = 256$}&\multicolumn{2}{c}{CIFAR-10, $d = 256$} \\
		Model&Clustering&Classification&Clustering&Classification\\
		\midrule
		Baseline Autoencoder  &\hfil$11.96$&\hfil$25.95$ &\hfil$21.73$&\hfil $40.92$\\
		VAE    & \hfil$13.58$ &\hfil$26.42$&\hfil$\textbf{24.12}$&\hfil $38.83$\\
		$\beta$-VAE ($\beta$ = 100) &\hfil $14.54$ &\hfil$49.62$&\hfil$22.80$&\hfil $36.91$\\
		Information dropout&\hfil$12.75$&\hfil$24.46$&\hfil$21.96$&\hfil$39.89$\\
		VQ-VAE (K = 512) &\hfil$12.96$ &\hfil$31.57$ &\hfil$20.30$&\hfil$33.51$ \\
		Soft VQ-VAE (K = 32)&\hfil$\textbf{17.68}$&\hfil\textbf{50.48}&\hfil23.83&\hfil\textbf{44.54}\\
		\bottomrule
	\end{tabular}}
	\label{sample-table}
\end{table}

We test $64$-dimensional latents for the MNIST and $256$ for SVHN and CIFAR-10. We compare different models where only the bottleneck operation is different. The results are shown in Table \ref{MNIST_table} and \ref{sample-table}. We report the means of accuracy results. The variances of all the results are within $1$ percent. 

For MNIST, soft VQ-VAE achieves the best accuracy for both clustering and classification tasks. Specially, it improves $25$ percent clustering accuracy for linear assignment metric and $36$ percent clustering accuracy for NMI metric when compared to the baseline autoencoder model. The performance of vanilla VQ-VAE suffers from the small size of the codebook ($K = 128$). All models show difficulties for directly learning from CIFAR-10 and SVHN data as they just perform better than random results in the clustering tasks. Soft VQ-VAE has the best accuracy for classification and has the second best for clustering. One reason for the poor performance of colored images may be that autoencoder models may need the color information to be dominant in the latent representation such that they can have a good reconstruction. However, the color information may not generally useful for clustering and classification tasks.

An interesting observation from the experiments is that we need to use a smaller codebook ($K = 32$) for the soft VQ-VAE for CIFAR-10 and SVHN when compared to MNIST ($K = 128$). According to our experiments, setting a larger $K$ for CIFAR-10 and SVHN will degrade the performance significantly. The potential reason is that we use CNN networks for CIFAR-10 and SVHN to have a better reconstruction of the colored images. Compared to the MLP networks used on MNIST, the CNN decoder is more powerful and can recover the encoder input from more cluttered latent representations. As a result, we need to reduce the codebook size to enforce a stronger regularization of the latents.

Beyond the discussed regularization effects, one intuition of the improved performance by soft VQ-VAE is that the embedded Bayesian estimator removes effects of adversarial input datapoint on the training. The adversarial points of the input data tend to reside in the boundary between classes. When training with ambiguous input data, the related codewords will receive a similar update. On the other hand, only one codeword receives a gradient update in the case of a hard assignment. This causes a problem. Ambiguous input is more likely estimated wrongly and the assigned codeword receives an incorrect update. Furthermore, the soft VQ-VAE model learns the variance for each Gaussian distribution. The learned variances control the smoothness of the latent distribution. The model will learn smoother distributions to reduce the effects of adversarial datapoints.
\section{Conclusion}
In this paper, we propose a regularizer that utilizes the quantization effects in the bottleneck. The quantization in the latent space can enforce a similarity-preserving mapping at the encoder. Our proposed soft VQ-VAE model combines aspects of VQ-VAE and denoising schemes as a way to control the information transfer. Potentially, this prevents the posterior collapse. We show the proposed estimator is optimal with respect to the bottleneck quantized autoencoder with noisy latent codes. Our model improves the performance of downstream tasks when compared to other autoencoder models with different bottleneck structures. Possible future directions include combining our proposed bottleneck regularizer with other advanced encoder-decoder structures \cite{DBLP:journals/corr/abs-1807-07543}\cite{DBLP:journals/corr/abs-1901-03416}. The source code of the paper is publicly available.\footnote{https://github.com/AlbertOh90/Soft-VQ-VAE/}
\section{Acknowledgements}
The authors sincerely thank Dr.~Ather Gattami at RISE and Dr.~Gustav Eje Henter for their valuable feedback on this paper.
We are also grateful for the constructive comments of the anonymous reviewers.
\small
\bibliographystyle{aaai}
\bibliography{fine3}
\end{document}